\documentclass[letterpaper,10pt,conference]{ieeeconf}
\IEEEoverridecommandlockouts
\usepackage{amsmath,amssymb}
\usepackage{bm}

\usepackage{amsthm}
\usepackage{cite}
\usepackage{graphicx}
\usepackage{url}
\usepackage{tikz}
\usepackage{tikz-qtree}
\usetikzlibrary{positioning,automata,backgrounds}
\usepackage{xcolor}
\usepackage{algorithm}
\usepackage[noend]{algpseudocode} 
\usepackage{subcaption} 
\usepackage{mathtools}
\usepackage{hyperref}
\usepackage{booktabs}
\usepackage{multirow}
\usepackage{bbm}
\usepackage{siunitx}
\usepackage{empheq} 
\usepackage{makecell}  

\theoremstyle{plain}
\newtheorem{theorem}{Theorem}
\newtheorem{lemma}{Lemma}
\newtheorem{proposition}{Proposition}

\newtheorem{definition}{Definition}

\newtheorem{problem}{Problem}
\newtheorem{remark}{Remark}

\newcommand{\xb}{\bm{x}}
\newcommand{\zb}{\bm{z}}
\newcommand{\bb}{\bm{b}}
\newcommand{\ub}{\bm{u}}

\newcommand{\ob}{\bm{o}}
\newcommand{\obpt}{\tilde{\ob}}

\newcommand{\fb}{\bm{f}}
\newcommand{\gb}{\bm{g}}
\newcommand{\Deltab}{\bm{\Delta}}
\newcommand{\indicator}{\mathbbm{1}}

\begin{document}

\title{Risk-Aware Belief Control Barrier Functions over Random Finite Sets}

\author{
Shaohang Han$^{*1}$, Gang Chen$^{2}$, Yixi Cai$^{1}$, Ignacio Torroba$^{3}$, Ivan Stenius$^{3}$, \\ Patric Jensfelt$^{1}$, Javier Alonso-Mora$^{2}$, and Jana Tumova$^{1}$
\thanks{*Corresponding author: \texttt{shaohang@kth.se}}
\thanks{$^{1}$Department of Robotics, Perception and Learning, KTH Royal Institute of Technology, Stockholm, Sweden}
\thanks{$^{2}$Department of Cognitive Robotics, Delft University of Technology, Delft, The Netherlands}
\thanks{$^{3}$Division of Aerospace, Moveability and Naval Architecture, KTH Royal Institute of Technology, Stockholm, Sweden}
\thanks{This work was partially supported by the Wallenberg AI, Autonomous Systems and Software Program (WASP) funded by the Knut and Alice Wallenberg Foundation. This work was also partially supported by Digital Futures, Vinnova and FMV through the SHARCEX grant. The authors are also affiliated with Digital Futures.}
}

\maketitle
\thispagestyle{empty}
	
\begin{abstract}
Ensuring robot safety in unknown, dynamic environments is a fundamental requirement. It involves inferring the states of an unknown and time-varying number of moving objects from noisy, incomplete measurements. We address safe control under the induced multi-object state uncertainty with a risk-aware belief control barrier function (BCBF) framework. The uncertainty is captured by a random finite set (RFS) belief, estimated by a sequential Monte Carlo probability hypothesis density (SMC-PHD) filter that represents it with a set of particles. Building directly on these particles, we construct a nonsmooth BCBF, establish forward invariance of the safe set under continuous prediction, and derive an explicit condition under which discrete updates preserve safety. Simulation and real-world underwater experiments demonstrate the effectiveness and efficiency of the proposed approach.
\end{abstract}

\section{Introduction}
Safety is a fundamental requirement in robotics, often formulated as a \textit{set invariance} problem: keeping the system state within a safe set that does not intersect the failure set \cite{cohen2024safety}. Control Barrier Functions (CBFs) provide a principled framework for synthesizing controllers that render safe sets forward invariant, typically via a quadratic program (CBF-QP) \cite{cohen2024safety,ames2019control,glotfelter2017nonsmooth}. Standard CBFs assume perfect state information \cite{ames2019control,glotfelter2017nonsmooth}, but real-world robotic systems operate under state uncertainty arising, e.g., from noisy and incomplete measurements. An effective way to reason about such uncertainty is through Bayesian inference, which provides a \textit{belief}: a probability distribution over possible states \cite{thrun2006probabilistic}. Building on this idea, recent work has introduced belief CBFs (BCBFs), which ensure safety using the beliefs produced by Kalman filters (KFs) and particle filters (PFs) \cite{10310096,10611412,11312006}.

These BCBFs assume the environment has a known and fixed number of objects.  In practice this is restrictive: limited sensor range and occlusions cause objects to enter and leave the scene. We therefore consider settings having an unknown and time-varying number of \emph{indistinguishable} moving objects.
This multi-object setting introduces two new challenges: the number of objects varies, and measurement-to-object associations are ambiguous. Both challenges are better handled by a belief over a random finite set (RFS), whose cardinality and elements are random.
A commonly used RFS-based estimator is the probability hypothesis density (PHD) filter. It approximates the multi-object belief as a Poisson point process (PPP) and propagates its first-order moment \cite{mahler2004multitarget}. In particular, the sequential Monte Carlo implementation of the PHD (SMC-PHD) filter \cite{vo2005sequential,ristic2010improved} has proven effective in nonlinear, non-Gaussian, and complex robotic environments \cite{chen2024continuous,ristic2010sensor,dames2017detecting}.

Despite the advantages of the SMC-PHD filter as a state estimator, incorporating its beliefs into a CBF framework remains challenging. 
First, state space safety specifications do not directly translate to belief space, where safety should be reasoned about probabilistically over distributions of states. This requires us to design a CBF in belief space with risk-aware guarantees.
Second, the belief space dynamics are hybrid in nature, combining continuous prediction with discrete update, which makes the analysis of forward invariance difficult.
Third, SMC-PHD filters typically represent the belief by a large number of particles \cite{vo2005sequential,ristic2010improved,chen2024continuous}, yielding a high-dimensional representation that complicates the synthesis of computationally efficient controllers.

\begin{figure}[t]
\centering
\includegraphics[width=0.9\linewidth]{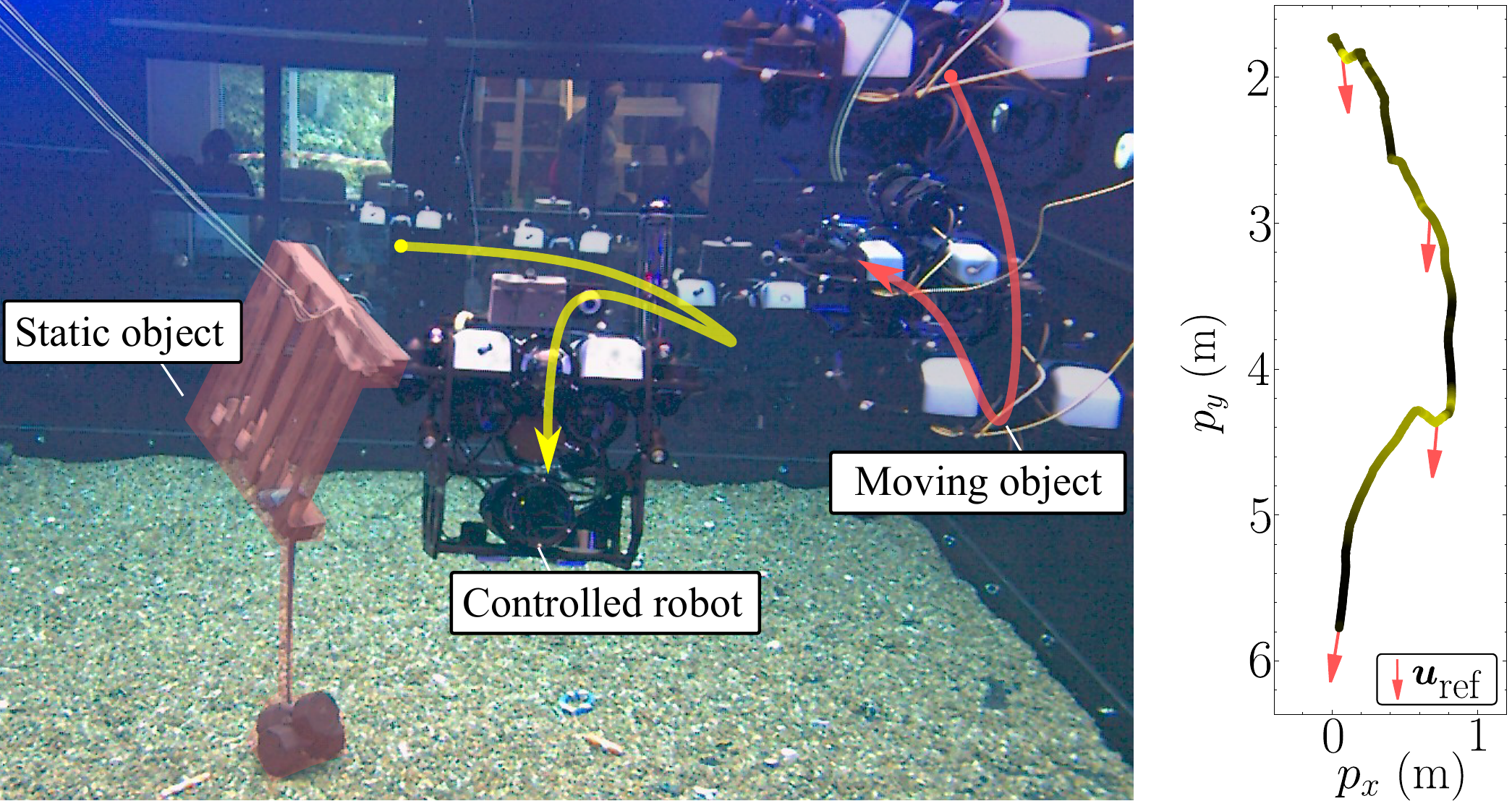}
\captionsetup{font=footnotesize}
\caption{Real-world experiment within a test tank. \textbf{Left:} A BlueROV2 equipped with a 
\href{https://www.waterlinked.com/shop/wl-21045-2-sonar-3d-15-689}{3D sonar}
must keep a safe distance from both the static wooden laths (highlighted with a red mask) and another moving ROV. \textbf{Right:} The recorded trajectory. A teleoperator commands the robot along the $-y$ axis ignoring the obstacles, while the BCBF-QP modifies this reference to ensure safety. The heatmap encodes the modification magnitude, with lighter yellow indicating larger changes. Red arrows show the reference velocity at selected timestamps. 
}
\vspace{-5mm}
\label{fig:hardware_validation}
\end{figure}

To address these challenges, we leverage the PPP structure of the PHD belief to construct a risk-aware BCBF directly on the particles. We then derive the corresponding nonsmooth CBF conditions \cite{glotfelter2017nonsmooth}, establish forward invariance under continuous prediction, and give an explicit condition for certifying safety across discrete filter updates. To handle the high-dimensional particle representation, we reformulate the CBF conditions into a more tractable form and exploit parallelism in constructing the BCBF-QP, achieving an average control computation time of less than $5$~ms. We validate our approach in simulation on two safe robot control applications: field-of-view (FOV) maintenance and obstacle avoidance in unstructured dynamic environments. We further demonstrate the latter on underwater hardware experiments, as shown in Fig.~\ref{fig:hardware_validation}.

\section{Related Work}
\subsection{CBFs under State Uncertainty}
Existing work characterizes measurement noise from samples and develops CBFs that are robust to it, using distributionally robust optimization (DRO) \cite{long2026sensor} or learning-based approaches \cite{li2023robust}. These methods handle noisy sensing but assume access to the full state, which is typically only partially observed in practice. For instance, the CBF condition may depend on velocity while the robot only measures position. A natural remedy is to estimate the full state and account for the resulting state uncertainty in CBFs. Following this idea, prior works extend CBFs to handle bounded state estimation errors \cite{wang2022observer,dean2021guaranteeing}. Observer-based CBFs derive these bounds from the error dynamics \cite{wang2022observer}, whereas measurement-robust CBFs obtain them from data \cite{dean2021guaranteeing}. However, both methods enforce safety against the worst-case error within the bound, which can be conservative compared to incorporating unbounded stochastic uncertainty \cite{akella2025risk}. Recent work on BCBFs \cite{10310096,10611412,11312006} accounts for this stochastic uncertainty but remains limited to a known, fixed number of objects in the environment.



\subsection{PHD Filters in Robotics}
The PHD filter has been applied to robot localization \cite{atanasov2016localization}, mapping \cite{chen2024continuous}, simultaneous localization and mapping \cite{gao2020random}, and multi-object search and tracking \cite{ristic2010sensor,dames2017detecting}. For multi-object state estimation, unlike KFs and PFs, the RFS-based PHD filter does not require explicit measurement-to-object association. This is particularly advantageous when the number of objects is large \cite{dames2017detecting,chen2024continuous}. The trade-off is that individual object identities are not maintained. This is not restrictive for the safe control tasks we target, where the safety specification treats all objects identically.

\section{Preliminaries}
We consider a robot modeled by the control-affine dynamics of the form
\begin{equation}
\label{eq:sys_robot}
    \dot{\xb} = \fb(\xb) + \gb(\xb)\ub,
\end{equation}
where $\xb \in \mathcal{X} \subseteq \mathbb{R}^{n_x}$ is the robot state, $\ub \in \mathcal{U} \subseteq \mathbb{R}^{n_u}$ is the control input, and the drift and input vector fields are given by $\fb:\mathbb{R}^{n_x}\to\mathbb{R}^{n_x}$ and $\gb:\mathbb{R}^{n_x}\to\mathbb{R}^{n_x \times n_u}$.

The environment contains an unknown number of objects, sharing the same dynamical and measurement models:
\begin{align}
\label{eq:sys_obj}
    \dot{\ob} &= \bm{\xi}(\ob), \\
\label{eq:sys_obj_measurement}
    \zb &= \bm{\ell}(\xb, \ob, \bm{\nu}).
\end{align}
We let $\ob \in \mathcal{O} \subseteq \mathbb{R}^{n_o}$ denote the object state, with dynamics $\bm{\xi}:\mathbb{R}^{n_o}\to\mathbb{R}^{n_o}$. The observation is $\zb \in \mathcal{Z} \subseteq \mathbb{R}^{n_z}$, generated by the measurement function $\bm{\ell}:\mathcal{X} \times \mathcal{O} \times \mathbb{R}^{n_\nu} \to \mathbb{R}^{n_z}$ under measurement noise $\bm{\nu} \in \mathbb{R}^{n_\nu}$. 

\subsection{Nonsmooth Control Barrier Functions}
Consider the dynamics defined in~(\ref{eq:sys_robot}) and~(\ref{eq:sys_obj}). Let the set ${\mathcal{C} \subseteq \mathbb{R}^{n_x} \times \mathbb{R}^{n_o}}$ be the closed zero super-level set of a locally Lipschitz function $h:\mathbb{R}^{n_x} \times \mathbb{R}^{n_o} \to \mathbb{R}$, defined as
\begin{equation}
    \begin{aligned}
        \mathcal{C} &:= \{ (\xb,\ob) \in \mathcal{X} \times \mathcal{O} \mid h(\xb,\ob) \geq 0 \}, \\
        \partial \mathcal{C} &:= \{ (\xb,\ob) \in \mathcal{X} \times \mathcal{O} \mid h(\xb,\ob) = 0 \}.
    \end{aligned}
    \label{eq:safe_set}
\end{equation}

\begin{definition}[Safety]
A set $\mathcal{C}$ is forward invariant if for every initial condition $\left(\xb(0),\ob(0)\right) \in \mathcal{C}$, it holds that $\left(\xb(t),\ob(t)\right) \in \mathcal{C}$ for all $t \geq 0$. The systems~\eqref{eq:sys_robot} and~\eqref{eq:sys_obj} are safe with respect to $\mathcal{C}$ if $\mathcal{C}$ is forward invariant.
\end{definition}

If the systems~\eqref{eq:sys_robot} and~\eqref{eq:sys_obj} are safe with respect to $\mathcal{C}$, we refer to $\mathcal{C}$ as the \emph{safe set} and the function $h$ as a \emph{safety function}.
To synthesize a controller that keeps the system safe, we employ nonsmooth CBFs. 


\begin{definition}[Nonsmooth Control Barrier Function \cite{glotfelter2017nonsmooth}]
Let $\mathcal{C}$ be defined as in~\eqref{eq:safe_set}. $h$ is a nonsmooth control barrier function if there exists an extended class-$\mathcal{K}_{\infty}$ function $\alpha$ such that for all $(\xb,\ob)\in\mathcal{X}\times\mathcal{O}$,
\begin{equation*}
\label{eq:ncbf}
\sup_{\ub \in \mathcal{U}} \ \inf_{\zeta \in \partial h(\xb,\ob)}
\zeta^\top
\begin{bmatrix} \fb(\xb)+\gb(\xb)\ub \\ \bm{\xi}(\ob) \end{bmatrix}
\ge
-\alpha\bigl(h(\xb,\ob)\bigr),
\end{equation*}
where $\partial h(\xb,\ob)$ denotes the Clarke generalized gradient of $h$ at $(\xb,\ob)$. A detailed definition can be found in \cite{glotfelter2017nonsmooth}.
\end{definition}

\subsection{Random Finite Sets and PHD Filters}
\label{sec:rfs_phd}
When the robot operates in an unknown and dynamic environment, neither the number of objects in the environment nor their states are known a priori. We therefore employ RFSs to represent multi-object states. Formally, we let $\mathbb{O}$ denote the object RFS, and let $\ob^{(i)} \in \mathcal{O}$ denote the state of its $i$-th element, then we have $\mathbb{O}=\{\ob^{(1)},\ob^{(2)},\ldots,\ob^{(n)}\}$. The cardinality $n\in\mathbb{N}$ is itself a random variable. The first moment of a distribution over an RFS is called the PHD, denoted by $D(\ob)$. As a density over the single-object state space, the PHD satisfies, for any subset $\mathcal{F} \subseteq \mathcal{O}$,
\begin{equation}
\label{eq:phd_inte}
 \int_{\ob \in \mathcal{F}} D(\ob)\, d\ob = \mathbb{E}\bigl[|\mathbb{O} \cap \mathcal{F}|\bigr],
\end{equation}
where $|\cdot|$ denotes set cardinality, i.e., the integral of the PHD over $\mathcal{F}$ equals the expected number of objects in $\mathcal{F}$. 

The robot receives a set of sensor measurements $\mathbb{Z}=\{\zb^{(1)},\zb^{(2)},\ldots,\zb^{(m)}\}$, where the number of measurements $m\in\mathbb{N}$ varies over time. Using these measurements, the PHD filter~\cite{mahler2004multitarget} recursively estimates the multi-object posterior. 
Analogous to how a Gaussian distribution underlies the KF, the PHD filter approximates the multi-object posterior by a PPP and assumes the objects are independent. Since a PPP is fully characterized by its first order moment, the multi-object belief reduces to the PHD. In the SMC implementation~\cite{vo2005sequential,ristic2010improved}, the PHD is approximated by weighted particles:
\begin{equation}
\label{eq:phd_smc}
    D(\ob) \approx \sum_{i=1}^{L} w^{(i)} \delta(\ob - \obpt^{(i)}),
\end{equation}
where $L$ is the number of particles, $w^{(i)}$ and $\obpt^{(i)}$ are the estimated weight and state of the $i$-th particle, and $\delta(\cdot)$ denotes the Dirac delta function. As measurements arrive only at discrete time steps $t_1,\ldots,t_k,\ldots$, the PHD filter consists of continuous prediction over the interval $t\in [t_{k-1},t_k)$ and a discrete update at $t=t_k$. The discrete update incorporates new measurements, followed by particle birth, weight update, and resampling. After resampling, we obtain a particle set $\{(\obpt^{(i)}, w)\}_{i=1}^L$ with a uniform weight $w := \sum_{i=1}^{L} w^{(i)} / L$. We refer the reader to \cite{vo2005sequential,ristic2010improved} for further details.

\subsection{Poisson Point Process}
\label{sec:ppp}
We now introduce the PPP and its key properties \cite{last2018lectures}.
\begin{definition}[Poisson Point Process]
Let $N$ be a random process on $\mathbb{R}^n$ such that, for each set $\mathcal{F} \subseteq \mathbb{R}^n$, $N(\mathcal{F})$ denotes the random number of points that lie in $\mathcal{F}$. Then $N$ is a PPP with intensity $\lambda\colon \mathbb{R}^n \to \mathbb{R}_{\geq 0}$ if:
\begin{enumerate}
    \item For every set $\mathcal{F} \subseteq \mathbb{R}^n$, the number of points in $\mathcal{F}$ is a Poisson random variable with distribution
    \begin{equation}
    \label{eq:ppp_pdf}
        \Pr\bigl(N(\mathcal{F}) = m\bigr)
        = \frac{\Lambda(\mathcal{F})^m \exp\bigl(-\Lambda(\mathcal{F})\bigr)}{m!},
    \end{equation}
    where $\Lambda(\mathcal{F}) := \int_{\ob \in \mathcal{F}} \lambda(\ob)\, d\ob$,
    and $m \in \mathbb{N}$.
    \item For any $k$ disjoint sets $\mathcal{F}_1, \ldots, \mathcal{F}_k \subseteq \mathbb{R}^n$, the random variables
    $N(\mathcal{F}_1), \ldots, N(\mathcal{F}_k)$ are independent.
\end{enumerate}
\end{definition}

\begin{remark}
\label{rem:ppp_intensity}
In the SMC-PHD filter, the PPP intensity $\lambda$ is the PHD $D(\ob)$, approximated by the particles in~\eqref{eq:phd_smc}.
\end{remark}

Setting $m=0$ in~\eqref{eq:ppp_pdf} gives the void probability of the PPP:
\begin{equation}
\label{eq:void_probability}
\Pr[N(\mathcal{F})=0]=e^{-\Lambda(\mathcal{F})}.
\end{equation}
This is the probability that $\mathcal{F}$ contains no objects, which we use to reformulate the safety specification.

\section{Problem Formulation}
\label{sec:problem_formulation}
We consider a robot governed by~\eqref{eq:sys_robot} operating in an environment containing an unknown and time-varying number of indistinguishable objects, each evolving according to~\eqref{eq:sys_obj}. The robot state is assumed known for simplicity, so that the focus is on the objects' state uncertainty. The multi-object state is modeled as an RFS $\mathbb{O}$ with the standard assumption that the objects are independent \cite{mahler2004multitarget,vo2005sequential,ristic2010improved}.

In the single-object setting, safety with respect to an object $\ob$ is specified by a single-object smooth safety function $h_o$, with the safe set $\mathcal{C}_o := \{(\xb, \ob) \in \mathcal{X} \times \mathcal{O} : h_o(\xb, \ob) \geq 0\}$. In the multi-object setting, we use the same $h_o$ for every object in $\mathbb{O}$. The robot is therefore safe if $h_o(\xb, \ob^{(i)}) \geq 0$ for every $\ob^{(i)} \in \mathbb{O}$. Since $\mathbb{O}$ is an RFS, we require this safety specification to hold probabilistically, i.e., in a \emph{risk-aware} sense at level $\tau$, formalized as the following problem.

\begin{problem}
\label{prob:belief_safe_control}
Given the robot dynamics~\eqref{eq:sys_robot}, the object dynamics~\eqref{eq:sys_obj}, the single-object safe set $\mathcal{C}_o$, and the multi-object RFS $\mathbb{O}$, synthesize control inputs such that at any time
\begin{equation}
\label{eq:problem_chance_constraint}
    \Pr\!\big[\,h_o(\xb, \ob^{(i)}) \geq 0,\ \forall \ob^{(i)} \in \mathbb{O}\,\big] \geq 1 - \tau
\end{equation}
for all $t \geq 0$, where $\tau \in (0,1)$ is a user-specified risk level, while remaining close to a reference control input $\ub_{\mathrm{ref}}$. 
\end{problem}

\begin{remark}
$\tau$ is the maximum probability that the safety specification is violated. A lower risk level $\tau$ corresponds to a stricter safety requirement, trading conservativeness for stronger robustness.
\end{remark}

To deal with the chance constraint in Problem~\ref{prob:belief_safe_control}, we leverage the SMC-PHD filter, which represents the multi-object belief by a set of weighted particles. The probability in~\eqref{eq:problem_chance_constraint} is thus evaluated under the PHD belief. Using this representation, we cast Problem~\ref{prob:belief_safe_control} as a set invariance problem in belief space. Specifically, we aim to design a safe set over the particle-based belief with risk-aware guarantees and synthesize a safe controller that renders this set forward invariant.

\section{Risk-Aware Safe Set in Belief Space}

\subsection{Particle-Based Safety Specification}
Under the SMC-PHD filter, the multi-object belief is parameterized by a weighted particle set $\{(\obpt^{(i)}, w)\}_{i=1}^L$, where all particles share a uniform weight $w$ after resampling. From this we define the belief state:
\begin{equation*}
\label{eq:belief_state}
    \bb := \begin{bmatrix} \obpt^{(1)} & \dots & \obpt^{(L)} \end{bmatrix}^\top \in \mathcal{B} \subseteq \mathbb{R}^{L \cdot n_o}.
\end{equation*}

To express the safety specification in belief space, we exploit the PPP structure of the multi-object belief via its void probability~\eqref{eq:void_probability}. We first specify the relevant region by introducing the single-object failure set
\begin{equation*}
\label{eq:unsafe_region}
\mathcal{F}(\xb) := \left\{ \ob \in \mathcal{O} \;\middle|\; h_o(\xb, \ob) < 0 \right\},
\end{equation*}
the set of object states that render the robot unsafe. Because the same $h_o$ applies to every object, $\mathcal{F}$ is shared across all $\ob^{(i)} \in \mathbb{O}$.
We denote $N(\mathcal{F}(\xb))$ as the number of objects in the failure set, and the two probabilities are equal:
\begin{equation*}
\Pr\!\big[\,h_o(\xb, \ob^{(i)}) \geq 0,\ \forall \ob^{(i)} \in \mathbb{O}\,\big] = \Pr\!\big[N(\mathcal{F}(\xb))=0\big].
\end{equation*}
Under the PPP belief, applying~\eqref{eq:void_probability} yields the belief space safety specification:
\begin{equation}
\label{eq:problem_count_constraint}
    e^{-\Lambda(\mathcal{F}(\xb), \bb)}
    \ge 1-\tau.
\end{equation}
Since the intensity of the PPP is estimated by the particle set $\{(\obpt^{(i)}, w)\}_{i=1}^L$ as described in Remark~\ref{rem:ppp_intensity}, we have 
\begin{equation}
\label{eq:lambda_indicator}
\Lambda(\mathcal{F}(\xb), \bb) \approx \sum_{i=1}^L w \indicator\{h_o(\xb, \obpt^{(i)}) < 0\}.
\end{equation}
By further simplifying~\eqref{eq:problem_count_constraint}, we obtain the specification
\begin{equation}
\label{eq:problem_weight_fail}
\ln\!\left(\frac{1}{1-\tau}\right) - \Lambda(\mathcal{F}(\xb), \bb) \geq 0.
\end{equation}
Intuitively,~\eqref{eq:problem_weight_fail} requires that the total weight of the particles in the failure set not exceed a threshold determined by the risk level $\tau$. 

\begin{remark}
By~\eqref{eq:phd_inte}, specification~\eqref{eq:problem_weight_fail} is equivalent to $\mathbb{E}[N(\mathcal{F}(\xb))] \le \ln(1/(1-\tau))$, i.e., a bound on the expected number of objects in the failure set. While such a bound could also be stated directly, the chance constraint formulation gives it a risk-aware guarantee and a clear interpretation, with $\tau$ directly encoding the risk level.
\end{remark}

\subsection{Belief Control Barrier Function}
Although the safety specification~\eqref{eq:problem_weight_fail} is straightforward to derive, its left-hand side cannot directly serve as a safety function, as the indicator in~\eqref{eq:lambda_indicator} is discontinuous in $(\xb,\bb)$, so its gradient is undefined. We therefore reformulate~\eqref{eq:problem_weight_fail} into an equivalent form from which we construct a locally Lipschitz BCBF. 
To start, we define $s_i(\xb,\bb) := h_o\!\left(\xb, \obpt^{(i)}\right)$ for $i=1,\dots,L$, so that particle $i$ lies in the safe set when $s_i \ge 0$ and in the failure set when $s_i < 0$. We let $M(\xb,\bb) := \sum_{i=1}^L \indicator\{s_i < 0\}$ count the particles in the failure set, so that specification~\eqref{eq:problem_weight_fail} reads $wM \leq \ln(1/(1-\tau))$. Since $M$ is integer-valued and $w$ is a uniform weight, the specification is equivalent to $M \le k_\tau$, where
\begin{equation*}
\label{eq:k_threshold}
k_\tau :=
\min\!\left\{
L,\;
\left\lfloor
\frac{1}{w}\ln\!\left(\frac{1}{1-\tau}\right)
\right\rfloor
\right\},
\end{equation*}
or, equivalently, to requiring at least $L - k_\tau$ safe particles:
\begin{equation}
\label{eq:count_form}
\sum_{i=1}^L \indicator\{s_i \ge 0\} \ge L - k_\tau.
\end{equation}
We let $\mathcal{S} := \bigl\{ \mathcal{I} \subseteq \{1,\dots,L\} \mid |\mathcal{I}| = L - k_\tau \bigr\}$ denote the collection of all index subsets of cardinality $L - k_\tau$. Condition~\eqref{eq:count_form} holds if and only if at least one subset $\mathcal{I} \in \mathcal{S}$ has all of its particles in the safe set, which we can write as
\begin{equation}
\label{eq:safety_spec_reform}
\max_{\mathcal{I} \in \mathcal{S}} \ \min_{i \in \mathcal{I}} s_i(\xb,\bb) \ge 0.
\end{equation}
We let $s_{(1)} \le \cdots \le s_{(L)}$ denote the order statistics of $\{s_i\}_{i=1}^L$. Since the inner $\min$ in~\eqref{eq:safety_spec_reform} increases with each $s_i$, the outer $\max$ is attained by any subset $\mathcal{I}^\star$ indexing the top-$(L - k_\tau)$ values $s_{(k_\tau+1)},\dots,s_{(L)}$. To remove the nonsmoothness of the inner $\min$, we replace it with the smooth soft minimum~\cite{molnar2023composing} and define our BCBF:
\begin{equation}
\label{eq:belief_cbf}
h_b(\xb,\bb) := -\frac{1}{\kappa}
\ln\!\left(
\sum_{i\in \mathcal{I}^\star} e^{-\kappa s_{i}(\xb,\bb)}
\right),
\end{equation}
where $\kappa > 0$. Here, $h_b$ is an under-approximation of $\min_{i \in \mathcal{I}^\star} s_i$ and recovers it as $\kappa \to \infty$. The corresponding belief space safe set is:
\begin{equation*}
\label{eq:belief_safe_set}
\begin{aligned}
\mathcal{C}_b &:= \bigl\{\, (\xb,\bb) \in \mathcal{X} \times \mathcal{B} \,\big|\; h_b(\xb,\bb) \ge 0 \bigr\}, \\
\partial \mathcal{C}_b &:= \bigl\{\, (\xb,\bb) \in \mathcal{X} \times \mathcal{B} \,\big|\; h_b(\xb,\bb) = 0 \bigr\}.
\end{aligned}
\end{equation*}
In the following lemma, we show that $h_b$ is locally Lipschitz.

\begin{lemma}
\label{lem:hat_h_lipschitz}
$h_b$ is locally Lipschitz in $(\xb,\bb)$.
\end{lemma}
\begin{proof}
Since $h_o$ is locally Lipschitz, each $s_i$ is locally Lipschitz in $(\xb,\bb)$. The sorting operation $(s_1,\dots,s_L)\mapsto(s_{(1)},\dots,s_{(L)})$ is locally Lipschitz ~\cite{anil2019sorting}, so selecting its top-$(L-k_\tau)$ values $(s_{(k_\tau+1)},\dots,s_{(L)})$ is also locally Lipschitz. As the soft minimum is smooth, $h_b$ is locally Lipschitz as a composition of locally Lipschitz functions.
\end{proof}

To enforce forward invariance of $\mathcal{C}_b$, there remain two challenges. First, although $h_b$ is locally Lipschitz, it could be nonsmooth when particles tie at the threshold value $s_{(k_\tau+1)}$. We therefore adopt the nonsmooth CBF framework~\cite{glotfelter2017nonsmooth} to synthesize a safe controller. Second, the belief evolves through continuous prediction and discrete updates, which we analyze separately in the next section.

\section{Safe Controller Synthesis}

\subsection{Belief Dynamics}
Between filter updates, each particle propagates according to the object dynamics:
\begin{equation*}
\dot{\bb} =
\begin{bmatrix}
    \bm{\xi}(\obpt^{(1)})^\top & \cdots & \bm{\xi}(\obpt^{(L)})^\top
\end{bmatrix}^\top
:= \bm{\Xi}(\bb).
\end{equation*}
The particle weight $w$ remains constant during the continuous prediction since no new information is incorporated. We denote the joint dynamics of the robot and the objects in the belief space as
\begin{equation}
\label{eq:belief_dynamics}
    \bm{F}(\xb,\bb,\ub) :=
    \begin{bmatrix} \fb(\xb)+\gb(\xb)\ub \\ \bm{\Xi}(\bb) \end{bmatrix}.
\end{equation}

At discrete update times $\{t_k\}_{k \in \mathbb{N}}$, the SMC-PHD filter incorporates new observations $\mathbb{Z}_k$, inducing a discrete update:
\begin{equation}
\label{eq:belief_update}
\bb(t_k^+) = \Deltab_k\!\left(\bb(t_k^-), \mathbb{Z}_k\right),
\end{equation}
where $\Deltab_k$ can be discontinuous, as resampling can duplicate or remove particles, and birth introduces new particles depending on the measurements. The weight $w$ is also updated at $t_k$ but is omitted here for brevity. The overall belief dynamics form a hybrid system as follows:
\begin{equation*}
\label{eq:belief_hybrid}
\begin{cases}
\dot{\bb} = \bm{\Xi}(\bb), & t \in [t_{k-1}, t_{k}), \\
\bb^+ = \Deltab_k\!\left(\bb^-, \mathbb{Z}_k\right), & t = t_k.
\end{cases}
\end{equation*}
Here, $\bb^-=\bb(t_k^-)$ and $\bb^+=\bb(t_k^+)$ denote the belief states immediately before and after the update at $t_k$, respectively.

\begin{figure}[b]
\centering
\includegraphics[width=0.58\columnwidth]{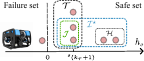}
\captionsetup{font=footnotesize}
\caption{An illustration of the tied scenario. The three tied particles (red circles) form $\mathcal{T}$, but only two of them enter the selection $\mathcal{I}^\star = \mathcal{H} \cup \mathcal{J}$. One possible choice of $\mathcal{J}$ is shown.}
\label{fig:illustration_tie_particles}
\end{figure}

\subsection{Safety under Continuous Prediction}
During the prediction phase, the belief evolves continuously according to~\eqref{eq:belief_dynamics}. When particles tie at the threshold value $s_{(k_\tau+1)}$, the selection $\mathcal{I}^\star$ may become non-unique, and where it does, the time derivative of $h_b$ differs across these selections, which makes $h_b$ nonsmooth. We therefore partition the indices into those strictly above the threshold, $\mathcal{H}:= \{ \eta \mid s_\eta > s_{(k_\tau+1)} \}$, and those tied at it, $\mathcal{T}:= \{ \mu \mid s_\mu = s_{(k_\tau+1)} \}$. Every $\mathcal{I}^\star$ retains all of $\mathcal{H}$ and completes it with a subset $\mathcal{J} \subseteq \mathcal{T}$ of cardinality $|\mathcal{J}| = L - k_\tau - |\mathcal{H}|$, i.e., $\mathcal{I}^\star = \mathcal{H} \cup \mathcal{J}$. An illustration is shown in Fig.~\ref{fig:illustration_tie_particles}. We collect all such selections in the \emph{active set}
\begin{equation*}
\mathcal{A} := \{\, \mathcal{H} \cup \mathcal{J} \mid \mathcal{J} \subseteq \mathcal{T},\ |\mathcal{J}| = L - k_\tau - |\mathcal{H}| \,\}.
\end{equation*}
To render $\mathcal{C}_b$ forward invariant, the following CBF condition should hold for every active realization:
\begin{equation}
\label{eq:cbf_condition}
\nabla h_b^{(\mathcal{I}^\star)\top}
\bm{F}(\xb, \bb, \ub)
\;\geq\; -\gamma h_b, \quad \forall\, \mathcal{I}^\star \in \mathcal{A},
\end{equation}
where $\gamma$ is a positive constant. Enforcing~\eqref{eq:cbf_condition} directly requires one constraint for each realization $\mathcal{I}^\star$, but the number of such realizations grows exponentially with the number of tied particles $|\mathcal{T}|$. Therefore, we exploit the structure of the tie to obtain a more conservative but tractable set of conditions whose size scales linearly in $|\mathcal{T}|$. We begin by expanding the gradient of each realization as
\begin{equation}
\label{eq:gradient_expansion}
\nabla h_b^{(\mathcal{I}^\star)} = \sum_{i \in \mathcal{I}^\star} c_i \nabla s_i = \sum_{\eta \in \mathcal{H}} c_\eta \nabla s_\eta + \sum_{j \in \mathcal{J}} c_j \nabla s_j.
\end{equation}
where the coefficient $c_i := e^{-\kappa s_i} / \sum_{\ell \in \mathcal{I}^\star} e^{-\kappa s_\ell}$.
Our key observation is that the tied particles share a common coefficient. Since $s_\mu = s_{(k_\tau+1)}$ for every $\mu \in \mathcal{T}$, each $c_\mu$ equals the same value $\bar{c}$. Consequently, $\sum_{j \in \mathcal{J}} c_j \nabla s_j = \bar{c} \sum_{j \in \mathcal{J}} \nabla s_j$ in~\eqref{eq:gradient_expansion}. This motivates us to introduce a single auxiliary variable $\rho \in \mathbb{R}$ that lower-bounds each tied gradient term, replacing~\eqref{eq:cbf_condition} with the following sufficient conditions:
\begin{subequations}\label{eq:final_cbf_condition}
\begin{empheq}[left=\empheqlbrace\,]{align}
&\sum_{\eta \in \mathcal{H}} c_\eta
\nabla s_\eta^\top \bm{F}(\xb,\bb,\ub)
+ \bar{c} |\mathcal{J}| \rho
\ge -\gamma h_b, \label{eq:conditions-j} \\
&\nabla s_\mu^\top \bm{F}(\xb,\bb,\ub) \ge \rho,
\quad \forall \mu \in \mathcal{T}. \label{eq:conditions-target}
\end{empheq}
\end{subequations}
We note that~\eqref{eq:final_cbf_condition} contains $|\mathcal T|$ constraints in~\eqref{eq:conditions-target} plus one constraint~\eqref{eq:conditions-j}.
The following lemma establishes that~\eqref{eq:final_cbf_condition} is sufficient for~\eqref{eq:cbf_condition}.

\begin{lemma}
\label{lem:final_implies_cbf}
If there exist a control input $\ub\in\mathcal U$ and an auxiliary variable $\rho\in\mathbb{R}$ such that~\eqref{eq:final_cbf_condition} holds, then~\eqref{eq:cbf_condition} holds for every $\mathcal{I}^\star \in \mathcal{A}$.
\end{lemma}
\begin{proof}
Since $\bar{c}>0$ and~\eqref{eq:conditions-target} gives $\nabla s_\mu^\top \bm{F}\ge\rho$ for all $\mu\in\mathcal{T}$, for any $\mathcal{J}\subseteq\mathcal{T}$ we have
\begin{equation*}
\sum_{j \in \mathcal{J}} c_j \nabla s_j^\top \bm{F}
= \bar{c}\sum_{j \in \mathcal{J}} \nabla s_j^\top \bm{F}
\ge |\mathcal{J}| \bar{c} \rho.
\end{equation*}
Combining this with~\eqref{eq:conditions-j} yields
\begin{equation*}
\sum_{i \in \mathcal{H}} c_i \nabla s_i^\top \bm{F}
+ \sum_{j \in \mathcal{J}} c_j \nabla s_j^\top \bm{F}
\ge -\gamma h_b, \quad \forall \mathcal{J} \subseteq \mathcal{T}.
\end{equation*}
By~\eqref{eq:gradient_expansion}, the left-hand side equals $\nabla h_b^{(\mathcal{I}^\star)\top} \bm{F}$. Since every $\mathcal{I}^\star \in\mathcal{A}$ is of this form,~\eqref{eq:cbf_condition} holds.
\end{proof}

Building on this, the next theorem establishes that satisfying conditions in~\eqref{eq:final_cbf_condition} renders $\mathcal{C}_b$ forward invariant during continuous prediction.

\begin{theorem}
\label{thm:hat_h_ncbf}
Suppose there exist a control input $\ub\in\mathcal U$ and a variable $\rho \in \mathbb{R}$ such that~\eqref{eq:final_cbf_condition} holds. Any controller satisfying~\eqref{eq:final_cbf_condition} renders $\mathcal{C}_b$ forward invariant under the continuous prediction~\eqref{eq:belief_dynamics}.
\end{theorem}

\begin{proof}
By Lemma~\ref{lem:hat_h_lipschitz}, $h_b$ is locally Lipschitz. Away from tie points, $\mathcal{A}$ contains a single subset and $h_b$ is smooth, therefore $\partial h_b = \{\nabla h_b\}$. When tie points occur, by~\cite[Theorem~1]{glotfelter2017nonsmooth}, the generalized gradient of $h_b$ satisfies
\begin{equation*}
\label{eq:gen_gradient}
\partial h_b(\xb,\bb) = \operatorname{co}\!\bigl\{ \nabla h_b^{(\mathcal{I}^\star)}(\xb,\bb) \mid \mathcal{I}^\star \in \mathcal{A} \bigr\},
\end{equation*}
where $\operatorname{co}$ denotes the convex hull. Given that~\eqref{eq:final_cbf_condition} holds, by Lemma~\ref{lem:final_implies_cbf},~\eqref{eq:cbf_condition} holds. By~\cite[Theorem~3]{glotfelter2017nonsmooth}, $h_b$ is a valid nonsmooth CBF for~\eqref{eq:belief_dynamics}. Thus, $\mathcal{C}_b$ is forward invariant.
\end{proof}

\subsection{Safety under Discrete Updates}
At time steps when sensor measurements become available, the belief state is updated according to~\eqref{eq:belief_update}. This update may introduce particles into the failure set $\mathcal{F}$. The resulting belief change cannot in general be fully compensated for by the control input at the update time. For example, under limited FOV or occlusion, a previously unobserved object may enter the FOV and trigger particle birth inside $\mathcal{F}$. Nevertheless, we mitigate this issue by constructing~\eqref{eq:belief_cbf} with a tightened risk level $\tau-\epsilon$, where $0<\epsilon<\tau$. We write $\mathcal{C}_b^{\tau-\epsilon}$ and $\mathcal{C}_b^{\tau}$ for the safe sets constructed with the tightened and original risk levels, respectively. The resulting safety margin permits a bounded increase in the total weight of particles in $\mathcal{F}$, as quantified in the following proposition.

\begin{proposition}
\label{prop:jump_robustness}
Suppose that before the update, $(\xb,\bb^-)$ is in $\mathcal{C}_b^{\tau-\epsilon}$. If the update satisfies
\begin{equation}
\label{eq:discrete_update_condition}
\Lambda(\mathcal{F}, \bb^+) - \Lambda(\mathcal{F}, \bb^-) \le \ln(1+\frac{\epsilon}{1-\tau}),
\end{equation}
and $\kappa \to\infty$, then $(\xb,\bb^+)$ stays in $\mathcal{C}_b^{\tau}$.
\end{proposition}

\begin{proof}
If the update satisfies~\eqref{eq:discrete_update_condition}, then~\eqref{eq:problem_weight_fail} holds with risk level $\tau$, indicating that~\eqref{eq:safety_spec_reform} is satisfied. Therefore, when $\kappa$ approaches $\infty$,~\eqref{eq:belief_cbf} remains non-negative, which means the state $(\xb, \bb)$ remains within $\mathcal{C}_b^{\tau}$ under the update.
\end{proof}

\subsection{Controller Synthesis via BCBF-QP}
The final risk-aware safe controller is synthesized via the following QP:
\begin{equation*}
\begin{aligned}
\ub^* = &\arg\min_{\ub \in \mathcal{U}, \rho} (\ub - \ub_\mathrm{ref})^\top \bm{Q} (\ub - \ub_\mathrm{ref}) \\
 \text{s.t.} \;\; & \eqref{eq:final_cbf_condition}\\
\end{aligned}
\end{equation*}
Here, $\bm{Q} \succ 0$ is the symmetric weighting matrix. 
We note that the controller synthesis reduces to a QP with $n_u+1$ decision variables and $|\mathcal{T}|+1$ constraints, which remains tractable even when the particle-based belief state is high-dimensional.
The CBF condition~\eqref{eq:final_cbf_condition} is enforced using the tightened risk level $\tau-\epsilon$. The following remark discusses how the state is maintained in $\mathcal{C}_b$ under the belief dynamics.

\begin{remark}
Theorem~\ref{thm:hat_h_ncbf} ensures that a state starting in $\mathcal{C}_b^{\tau-\epsilon}$ remains in this set during continuous prediction. If the subsequent update satisfies~\eqref{eq:discrete_update_condition}, Proposition~\ref{prop:jump_robustness} ensures that the updated state lies in $\mathcal{C}_b^{\tau}$ as $\kappa\to\infty$. Keeping the state in $\mathcal{C}_b^{\tau}$ over successive prediction intervals and updates additionally requires it to return to $\mathcal{C}_b^{\tau-\epsilon}$ before the next update. This recovery property is not explicitly enforced by the current controller. While it might be addressed by a time-varying CBF formulation~\cite{lindemann2018control}, we leave it for future work.
\end{remark}

\section{Experiments and Applications}
We validate our approach in two example applications
\footnote{
Code will be released upon acceptance.
}, where the robot needs to estimate the multi-object state to satisfy the safety specification.

\subsection{Implementation Details}
For runtime efficiency, we exploit a structural property of~\eqref{eq:final_cbf_condition}: computing each $s_i$ and evaluating $c_i \nabla s_i^\top \bm{F}$ are independent across particles and therefore parallelizable. We implement this in JAX~\cite{bradbury2018jax} using vectorization and automatic differentiation. A further difficulty arises when evaluating $c_i$ in~\eqref{eq:gradient_expansion}, since $e^{-\kappa s_i}$ can underflow for large $\kappa$. To avoid this, we rescale the numerator and denominator by a common factor, which preserves $c_i$ while ensuring numerical stability. 

We solve the QPs with OSQP~\cite{stellato2018osqp}. To guarantee that a control input is always available, we relax the CBF constraint with a heavily penalized slack variable. 
In simulation, we enable the relaxation only as a fallback when the strict QP is infeasible. On hardware, we instead keep it always active to ensure real-time performance. Throughout, we use the same $\gamma$ for our method and the baselines, since a smaller $\gamma$ induces more conservative behavior across all CBFs.

All computations are performed on a laptop with an Intel Core i7-13700H CPU, 32 GB of RAM, and an NVIDIA RTX 4070 GPU with 8 GB of memory.

\subsection{Multi-Object FOV Maintenance}
\subsubsection{Setup}
We consider a planar robot modeled as a unicycle with state $\xb = [p_x, p_y, \theta]^\top \in \mathbb{R}^2 \times \mathrm{S}^1$ and dynamics
\begin{equation*}
\dot{p}_x = u_v\cos\theta, \quad \dot{p}_y = u_v\sin\theta, \quad \dot{\theta} = u_\omega,
\end{equation*}
where $[p_x,p_y]^\top$ and $\theta$ denote the position and heading of the robot, respectively, and $u_v$ and $u_\omega$ are the linear and angular velocity inputs. The robot is equipped with a forward-facing sensor with a limited FOV that returns a set of range-and-bearing measurements. The robot should keep the moving objects within its FOV, whereas a reference controller attempts to stabilize the robot at its initial pose $[0,0,\pi/2]^\top$. An unknown number of objects are initialized in the FOV, sharing the dynamics and measurement model:
\begin{equation*}
\begin{aligned}
&\dot{\bm{q}} = \bm{v}, \quad \dot{\bm{v}} = \bm{0}, \\
&\zb = [\sqrt{{}^p q_x^2 + {}^p q_y^2}, \ \mathrm{arctan2}({}^p q_y, {}^p q_x)]^\top + \bm{\nu},
\end{aligned}
\end{equation*}
where $\bm{q}=[q_x, q_y]^\top$ and $\bm{v}=[v_x, v_y]^\top$ are the object position and velocity, and $[{}^p q_x, {}^p q_y]^\top$ denotes the object position in the robot's local frame. $\boldsymbol{\nu} \sim \mathcal{N}(\bm{0}, \mathbf{R})$ is the measurement noise with $\mathbf{R} = \mathrm{diag}(\sigma_r^2, \sigma_b^2)$, where $\sigma_r = \SI{1}{\meter}$ and $\sigma_b = \SI{1}{\degree}$. The sensor measurement rate is $\SI{10}{\hertz}$. Since the velocity is not directly observed and the measurement model is nonlinear, we use the SMC-PHD filter \cite{ristic2010improved} to estimate the full multi-object state. We resample to a fixed $L=3000$ particles at each discrete update step. 

We model the FOV as a bounded angular sector with aperture $\beta=\SI{50}{\degree}$, as illustrated in Fig.~\ref{fig:exp_fov_A}. Following~\cite{catellani2023distributed}, we define a pair of single-object safety functions in the state space, one for each edge of the FOV sector:
\begin{equation*}
h^{\mathrm{R/L}}_o(\xb,\ob) := \tan(\beta/2)\cdot{}^p q_x \pm {}^p q_y.
\end{equation*}
We set the true number of objects to $N_o=4$. The objects' initial positions are randomly sampled along the central axis of the angular FOV sector, and their velocities are constant, with magnitudes uniform in $[0.5, 1]$, and directions uniform in $[-\pi/3, -\pi/6]$ relative to the positive $x$-axis. 

\subsubsection{Baselines}
Given the particles from the SMC-PHD filter, a common way to extract object state estimates is to first cluster the particles and then take the mean of each cluster~\cite{vo2005sequential,ristic2010improved}. We then evaluate the state space safety function $h^{\mathrm{R/L}}_o$ for each object state estimate and aggregate them into a single CBF~\cite{molnar2023composing}, which we refer to as Mean-CBF. Alternatively, one can take the highest-weight particle in each cluster as a maximum a posteriori (MAP) estimate. The CBF synthesized from these estimates is referred to as MAP-CBF.

\begin{figure}[t]
\centering
\includegraphics[width=0.8\columnwidth]{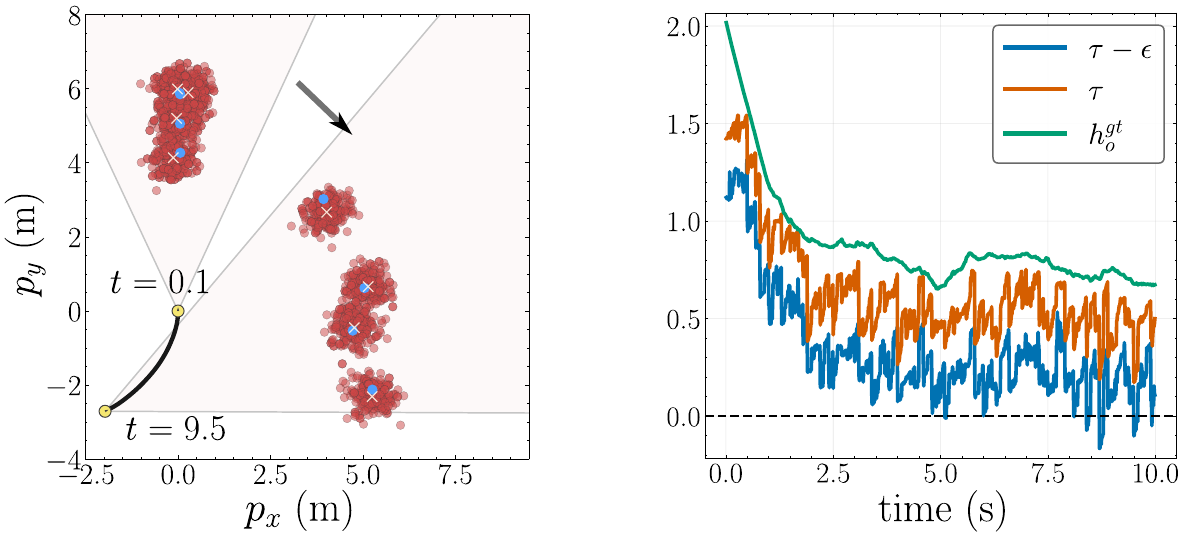}
\captionsetup{font=footnotesize}
\caption{ Multi-object FOV maintenance.
\textbf{Left}: Snapshots at $t=0.1\,\mathrm{s}$ and $9.5\,\mathrm{s}$, showing the robot (yellow dot) with its trajectory (black) and its FOV (pink sector), the true objects (blue dots), the particles (red dots), and the detections (white crosses). A black arrow indicates the objects' moving direction. \textbf{Right}: Evolution of $h^{\mathrm{gt}}_o$ and $\min \{ h^{\mathrm{L}}_b, h^{\mathrm{R}}_b \}$ with different risk levels over the simulation.
}
\label{fig:exp_fov_A}
\end{figure}

\begin{table}[t]
\centering
\captionsetup{font=footnotesize}
\caption{FOV maintenance results (mean $\pm$ std across random seeds)}
\label{tab:fov_performance_comparison}
\scriptsize
\renewcommand{\arraystretch}{0.8}
\begin{tabular}{@{}lcccc@{}}
\toprule
& \makecell{Ours \\ {\tiny($\tau{-}\epsilon{=}0.01$)}} & \makecell{Ours \\ {\tiny($\tau{-}\epsilon{=}0.2$)}} & Mean-CBF & MAP-CBF \\ \midrule
$\min\limits_t h^{\mathrm{gt}}_o(t)$
& $\mathbf{0.51}_{\pm 0.12}$
& $0.24_{\pm 0.12}$
& $-0.16_{\pm 0.54}$
& $0.05_{\pm 0.45}$ \\ \midrule
\# Unsafe & $\mathbf{0}$ & $5$ & $49$ & $31$ \\ \midrule
Avg.\ $t_{c}$ (ms)
& $1.72_{\pm 0.06}$
& $1.78_{\pm 0.06}$
& $0.60_{\pm 0.02}$
& $0.61_{\pm 0.02}$ \\ \midrule
Max.\ $t_{c}$ (ms)
& $6.84$
& $7.46$
& $3.90$
& $4.14$ \\ \bottomrule
\end{tabular}
\vspace{-5mm}
\end{table}

\subsubsection{Results}
We evaluate our method against the baselines over 100 simulations. The statistical results are reported in Table~\ref{tab:fov_performance_comparison}. 
At each time $t$, we let
\begin{equation*}
h^{\mathrm{gt}}_o(t) \;:=\; \min_{i \in \{1,\dots,N_o\}} \; \min\!\big\{\, h_o^{\mathrm{L}}(\xb, \ob^{(i)}),\; h_o^{\mathrm{R}}(\xb, \ob^{(i)}) \,\big\}
\end{equation*}
denote the smallest ground-truth value of the state space safety function over both FOV edges and all $N_o$ objects. If $h^{\mathrm{gt}}_o$ becomes negative, the objects are not within the FOV. As shown in the table, our BCBF keeps $h^{\mathrm{gt}}_o$ nonnegative when the risk level is sufficiently low, whereas both Mean-CBF and MAP-CBF can be unsafe. This is because they reduce the particle-based belief to a single estimate per object, discarding the spatial distribution of the particles. This leads to safety violations when the PHD belief is imperfect. In contrast, our BCBF is synthesized from the full belief and therefore remains robust to state uncertainty. Moreover, decreasing the risk level enlarges the safety margin, as the BCBF drives $h^{\mathrm{gt}}_o$ further from zero, yielding stronger robustness. The computation time $t_c$ includes both QP construction and solution. In all cases the average $t_c$ remains below $5\,\mathrm{ms}$, demonstrating the real-time performance of our method.
    
We visualize one simulation run of our BCBF with tightened risk level ($\tau=0.05$, $\epsilon=0.04$) in Fig.~\ref{fig:exp_fov_A}. The objects move from the upper-left toward the lower-right. To maintain the objects within the FOV, the BCBF controller rotates the robot and drives it backward. We also plot the evolution of $\min \{ h^{\mathrm{L}}_b, h^{\mathrm{R}}_b \}$. The observed spikes are mainly caused by the discrete PHD filter update. Notably, the BCBF value evaluated at the tightened level $\tau - \epsilon$ may become negative, yet the BCBF value evaluated at the original risk level ($\tau=0.05$) remains nonnegative throughout. This illustrates that tightening the risk level can permit moderate belief changes caused by the discrete updates.

\begin{figure}[t]
\centering
\includegraphics[width=0.8\columnwidth]{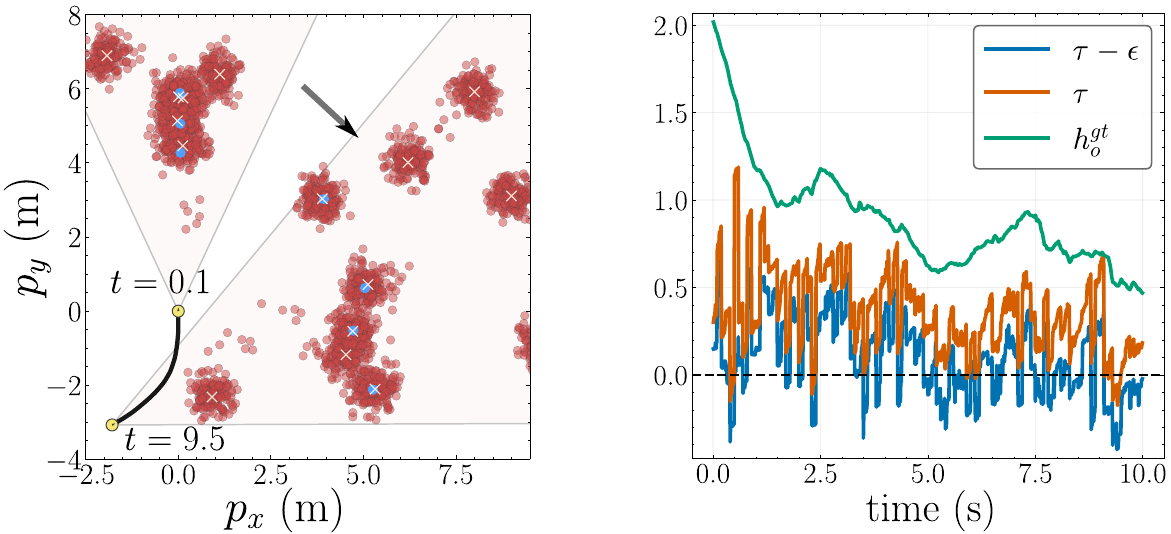}
\captionsetup{font=footnotesize}
\caption{ Same scenario as Fig.~\ref{fig:exp_fov_A}, but with a high false-alarm rate.
\textbf{Left}: Snapshots using the same conventions. The detections contain false alarms. \textbf{Right}: Same as Fig.~\ref{fig:exp_fov_A}.
}
\label{fig:exp_fov_B}
\vspace{-5mm}
\end{figure}

In practice, sensor measurements could contain false alarms caused by complex environmental effects, such as sand ripples in underwater scenes~\cite{williams2011adaptive}. To further demonstrate the robustness of our framework, we repeat the simulation with false alarms, modeled as a PPP with uniform intensity. The results are shown in Fig.~\ref{fig:exp_fov_B}. The false alarms induce additional particle birth, causing more frequent spikes in the BCBF values, occasionally driving them negative. Nevertheless, the PHD filter is still able to estimate the objects' states, and the safe controller keeps $h^{\mathrm{gt}}_o$ nonnegative throughout, maintaining the objects within the FOV.

\subsection{Obstacle Avoidance in Dynamic Environments}

\subsubsection{Simulation Setup and Baselines}
We consider the robot to be a 3D single integrator with state $\xb = [p_x, p_y, p_z]^\top \in \mathbb{R}^3$ denoting its position, and dynamics $\dot{\xb} = \ub$, where $\ub$ is the input velocity. The robot carries a depth sensor that returns position point clouds of the environment, with no prior knowledge of the obstacles' shapes. We build on the SMC-PHD in \cite{chen2024continuous}, which treats the point cloud directly as a set of objects and estimates their positions and velocities, yielding a unified representation of the environment and its motion. The trade-off is a large, time-varying object set, since points appear and disappear under sensor motion and occlusion --- precisely the setting where the RFS formulation is effective \cite{chen2024continuous}. Following this approach, the object dynamics and measurement model are
\begin{equation*}
    \dot{\bm{q}} = \bm{v}, \quad \dot{\bm{v}} = \bm{0}, \quad \zb = {}^p\bm{q} + \bm{\nu},
\end{equation*}
where $\bm{q}=[q_x, q_y, q_z]^\top$ and $\bm{v}=[v_x, v_y, v_z]^\top$ are the object position and velocity, ${}^p\bm{q}$ is the position in the robot's local frame, and $\bm{\nu} \sim \mathcal{N}(\boldsymbol{0}, \boldsymbol{R})$ is the measurement noise. We use PyBullet~\cite{coumans2021} to simulate the robot and the obstacles. The point cloud measurement is generated by ray-casting uniformly around the robot at a \SI{10}{\hertz} rate. In simulation, we assume that the robot has a full spherical FOV. We resample to a fixed $L=8000$ particles for the SMC-PHD filter. 

The single-object state space safety function is defined as 
\begin{equation*}
    h_o(\xb,\ob) := \|\,[p_x - q_x, p_y - q_y, p_z - q_z]\|_2 - r_{\mathrm{safe}},
\end{equation*}
where $r_{\mathrm{safe}}$ is the safe distance. We benchmark our BCBF against a CBF that uses the soft minimum distance between the robot and the point cloud~\cite{harms2025safe}, which considers point positions but not their motion, i.e., it implicitly treats $\bm{v} = \bm{0}$. Both methods use $\kappa=100$ for the soft minimum.

\begin{figure}[t]
\centering
\includegraphics[width=1.0\columnwidth]{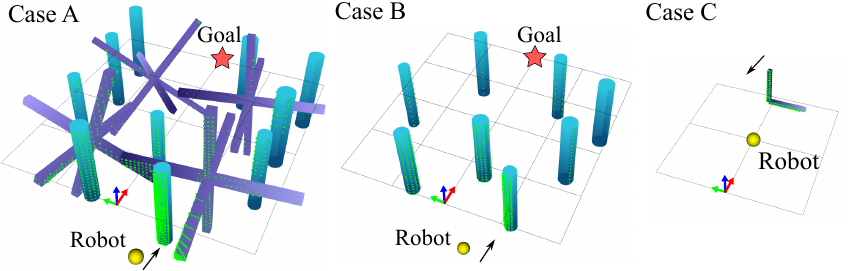}
\captionsetup{font=footnotesize}
\caption{Illustration of the obstacle avoidance settings in simulation, showing the robot (yellow), the obstacles (blue), and the depth point clouds (green).
\textbf{Case A:} static cross-shaped obstacles and cylinders moving in random directions at \SI{2.5}{\meter\per\second}.
\textbf{Case B:} a head-on scenario where cylindrical obstacles move toward the robot at $[-2.5, 0, 0]^\top$.
In both cases A and B, a reference controller drives the robot to the goal.
\textbf{Case C:} a head-on scenario where an L-shaped obstacle moves toward the robot at $[-1.0, 0, 0]^\top$, and a reference controller attempts to hold the initial position.
}
\label{fig:exp_obstacle_sce}
\end{figure}

\begin{table}[t]
\centering
\scriptsize
\captionsetup{font=footnotesize}
\caption{Performance comparison in different cases. The value in parentheses after ``Ours'' denotes the risk level. The success and collision rates need not sum to 100\%, as some runs time out.}
\label{tab:performance_comparison}
\setlength{\tabcolsep}{2pt}
\renewcommand{\arraystretch}{0.9}
\begin{tabular}{c|cccccc}
\toprule
    & Method & Coll.\ (\%) & Succ.\ (\%) & $t_n$ (s) & Avg. $t_c$ (ms) & Max.\ $t_c$ (ms) \\
\midrule
\multirow{3}{*}{A}
    & Ours ($0.05$) & $\mathbf{1.0}$ & $94.0$ & $4.47_{\pm 1.22}$ & $2.66_{\pm 0.13}$ & $10.70$ \\
    & Ours ($0.15$) & $\mathbf{1.0}$ & $\mathbf{95.0}$ & $4.38_{\pm 1.29}$ & $2.65_{\pm 0.12}$ & $9.10$ \\
    & \cite{harms2025safe} & $19.0$ & $80.0$ & $3.47_{\pm 1.04}$ & $1.51_{\pm 0.06}$ & $5.02$ \\
\midrule
\multirow{3}{*}{B}
    & Ours ($0.05$) & $\mathbf{0.0}$ & $92.0$ & $4.09_{\pm 1.72}$ & $2.43_{\pm 0.08}$ & $9.26$ \\
    & Ours ($0.15$) & $\mathbf{0.0}$ & $\mathbf{93.0}$ & $4.04_{\pm 1.44}$ & $2.42_{\pm 0.08}$ & $9.50$ \\
    & \cite{harms2025safe} & $24.0$ & $76.0$ & $2.43_{\pm 0.14}$ & $1.40_{\pm 0.04}$ & $4.74$ \\
\midrule
\multirow{3}{*}{C}
    & Ours ($0.05$) & $\mathbf{0.0}$ & $-$ & $-$ & $2.38_{\pm 0.04}$ & $8.44$ \\
    & Ours ($0.15$) & $1.0$ & $-$ & $-$ & $2.38_{\pm 0.04}$ & $8.67$ \\
    & \cite{harms2025safe} & $64.0$ & $-$ & $-$ & $1.38_{\pm 0.04}$ & $7.13$ \\
\bottomrule
\end{tabular}
\vspace{-5mm}
\end{table}

\subsubsection{Simulation Results}
We evaluate both methods on three test cases shown in Fig.~\ref{fig:exp_obstacle_sce}. In particular, case A presents a highly unstructured environment, mixing arbitrarily shaped static structures with moving obstacles. For cases A and B, we randomize the obstacles' initial poses and velocity directions. For case C, we only keep the randomness in the evolution of the PHD filter. We run 100 simulations for each case, and the results are reported in Table~\ref{tab:performance_comparison}. 

In cases A and B, our BCBF achieves a higher success rate and a lower collision rate while maintaining a reasonable time $t_n$ to reach the goal, reflecting the additional caution from anticipating obstacle motion. In contrast, the baseline~\cite{harms2025safe} ignores obstacle velocities and thus incurs more collisions. In case C, our BCBF again attains a lower collision rate. As shown by Fig.~\ref{fig:traj_caseC}, our method reacts earlier and avoids the obstacle, whereas the baseline collides. 

\begin{figure}[h]
\centering
\vspace{-3mm}
\includegraphics[width=0.8\columnwidth]{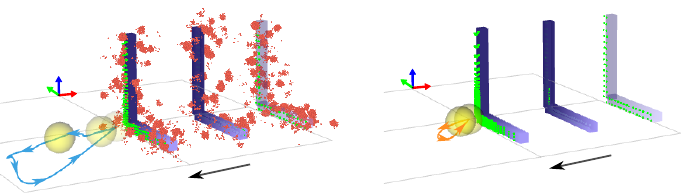}
\captionsetup{font=footnotesize}
\caption{Snapshots of one simulation run of case C at $t=2.4, 2.9, 3.4$, showing the particles of our method (red), the depth point clouds (green). The trajectories of our method and the baseline are shown in blue and orange, respectively. \textbf{Left:} Ours. \textbf{Right:}~\cite{harms2025safe}.
}
\label{fig:traj_caseC}
\vspace{-3mm}
\end{figure}

\subsubsection{Hardware Validation}
We validate our method on a BlueROV underwater robot equipped with a 3D sonar, with state estimation from motion capture measurements in a tank arena~\cite{torroba2026marinarium}. The sonar provides point cloud measurements at $\SI{5}{\hertz}$ within a $\SI{90}{\degree} \times \SI{40}{\degree}$ (vertical $\times$ horizontal) FOV, while the safe controller runs at $\SI{50}{\hertz}$. We follow the reduced-order hierarchical control framework of~\cite{cohen2024safety}, where a low-level controller tracks the input velocities, enabling us to give safe velocity commands at the kinematic level. 

We first recreate the head-on collision setting of case C in hardware, shown in Fig.~\ref{fig:hardware_headon}.
As in simulation, our method reacts early enough by anticipating the obstacle's motion, whereas the baseline~\cite{harms2025safe} becomes unsafe. Then, we test our method under adversarial teleoperation toward the obstacles at about $\SI{0.2}{\meter\per\second}$, as shown in Fig.~\ref{fig:hardware_validation}. Our safe controller modifies the reference to keep a safe distance from both static and moving obstacles, demonstrating real-world applicability in dynamic unstructured environments.

\begin{figure}[t]
\centering
\includegraphics[width=0.9\columnwidth]{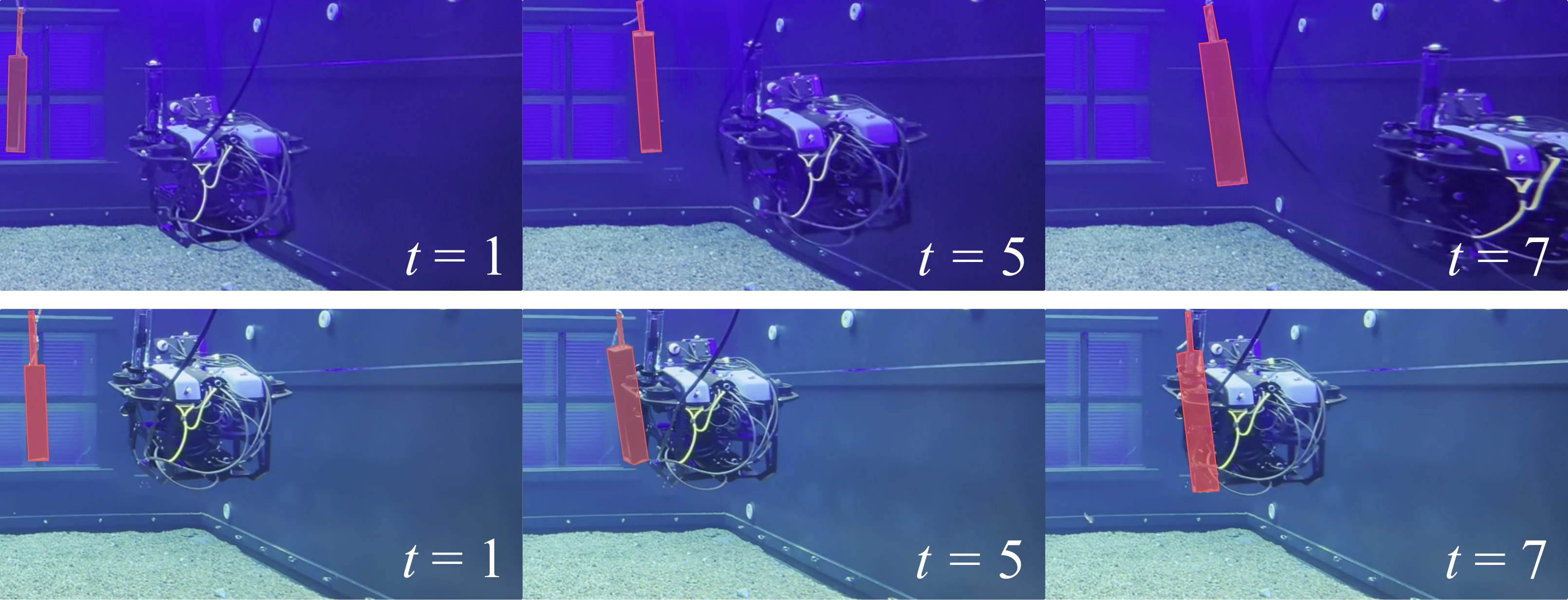}
\captionsetup{font=footnotesize}
\caption{Snapshots of one hardware run in the head-on collision scenario. The obstacle is highlighted with a red mask. \textbf{Upper:} Ours. \textbf{Lower:}~\cite{harms2025safe}.
}
\label{fig:hardware_headon}
\vspace{-5mm}
\end{figure}

\section{Conclusion}
In this work, we proposed a risk-aware BCBF framework for the particle-based beliefs produced by the SMC-PHD filter. The resulting safe controller remains computationally efficient even when the particle-based belief state is high-dimensional. We validated its robustness and adaptability in both simulation and hardware experiments. For future work, we intend to account for the mismatch between the estimated PPP belief and the true multi-object posterior, for example, through distributionally robust formulations~\cite{11312006,long2026sensor}. We also plan to move the computation onboard, enabling untethered safe operation in underwater environments.

\section*{Acknowledgement}
{\small We thank Matti Vahs, Ming Li, and Luzia Knoedler for the helpful discussions, and Julian Valdez for the hardware experiments.}
    
\bibliographystyle{IEEEtran}
\bibliography{references}

\end{document}